\documentclass[sigconf,nonacm]{acmart}

\makeatletter
\let\@orig@maketitle\maketitle
\renewcommand{\maketitle}{%
  \@orig@maketitle
  \fancyhead[RE]{\@headfootfont\shorttitle}%
}
\makeatother

\usepackage{booktabs}
\usepackage{amsmath}
\usepackage{tabularx}
\usepackage{multirow}
\usepackage{enumitem}
\usepackage{float}
\usepackage{tikz}
\usetikzlibrary{arrows.meta,positioning,shapes.geometric,decorations.pathreplacing,calc}
\usepackage{pgfplots}
\pgfplotsset{compat=1.18}
\usepackage{listings}
\usepackage{xurl} 

\newtheorem{proposition}{Proposition}[section]

\title{Bounded Agents: Delegation Security for Multi-Agent AI Systems}

\author{Xabier Muruaga}
\email{xabier@muruaga.ai}
\affiliation{%
  \institution{Independent Researcher}
  \country{}
}

\begin{document}

\begin{abstract}
LLM-based agents can act on behalf of a user to access cloud services, call tools, or invoke agents. At session start, the agent's permissions are set but remain static, and each request is evaluated independently, without considering prior actions. Within its permissions, an agent may act contrary to the delegated task, combine individually permitted actions into a prohibited outcome, or delegate authority to a sub-agent without limiting it. A prompt injection poses a risk only if the agent has authority to perform such actions; this is therefore a problem of authorization architecture, not just the model. The Agentic Principal Chain (APC) tracks delegated authority from one principal to the next. APC evaluates each request against the accumulated session state using six authorization checks. APC carries forward and restricts delegated scope and budgets. Using composition closure, APC checks requests against prior actions to prevent prohibited combinations and enforces the decision outside the model. We prove Blast Radius Monotonicity and Composition Soundness for APC implementations; Composition Soundness is limited to prohibited combinations under a complete restriction set and serialized admission. We evaluated 3,154 instances including InjecAgent, AgentDojo, and ASB. Our compromised-model evaluation tests APC independently of model behavior by inserting the ground-truth attack call after the first legitimate tool call. AgentDojo exfiltration fell from 75--100\% to 0\% across all four domains; APC blocked all 544 InjecAgent data-stealing cases. Intent binding reduced destruction from 38.6\% to 4.0\% and manipulation from 90.5\% to 12.1\%. Authorization latency was 0.24\,ms at the 99th percentile on an idle host; across 949 AgentDojo task--injection pairs, utility was 8.6 and 13.9 percentage points lower in the two settings. Implementation, evaluation tools, and data are publicly available.
\end{abstract}


\maketitle

\section{Introduction}\label{sec:intro}

Enterprise systems progressively deploy LLM-based agents that plan tasks, invoke tools, delegate to sub-agents, and produce irreversible effects whose execution paths are determined at runtime. The access-control methods supervising these systems were designed for a setting in which the acting entity is human, the delegation is explicit, and the granted scope is static. None of these assumptions hold for agentic AI: the acting entity is a non-human identity driven by a probabilistic model, delegation is dynamic and recursive across sub-agents, and the effective scope of a session evolves as the agent invokes tools.

This paper develops one central argument: \textbf{the security consequence of prompt injection is an authorization-architecture problem, not solely a model-robustness problem.} An agent with access to an external communication tool can be induced to exfiltrate documents; an agent without that access cannot, regardless of what is injected into its context~\cite{ruan2024, debenedetti2024}. The relevant question is not only whether the model follows a malicious instruction---it is what the model is \emph{authorized} to do when it does. We formalize this argument as the \textbf{Agentic Principal Chain (APC)} model: a session-scoped authorization state carried across delegated agent workflows and enforced by infrastructure outside the model runtime.

From this central argument follow three structural observations.

First, \textbf{the model must be excluded from its own trust boundary.} Most deployments place security controls at the model layer: ``always ask before deleting.'' These are prompt instructions that can be overridden or ignored. An approval gate implemented as a prompt instruction is not equivalent to one implemented in a Policy Enforcement Point (PEP).

Second, \textbf{per-action authorization is structurally insufficient.} An agent authorized to read confidential documents and to send external email can combine both to exfiltrate data---without violating any individual permission. This is a confused-deputy problem at the level of action composition. Agentic systems require \textbf{composition closure}: formal constraints on which action combinations are prohibited, evaluated over the session history, not per action in isolation.

Third, \textbf{agentic systems must assume breach.} Some component will be compromised. The question is how much damage it can cause. Every delegation hop introduces another point at which authority may be compromised, and the blast radius must be bounded by architecture.

A scope delimitation: \textbf{APC addresses the authorization architecture---delegation-safe action admissibility and composition closure---not the full surface of agent security.} Parameter-level validation of individual tool calls is a complementary problem outside APC's design boundary. Defenses against model manipulation, tool-specific policy engines, and data-flow isolation address different security problems; APC is designed to compose with, not replace, these mechanisms, and to layer on top of existing identity and policy infrastructure (e.g., OAuth and policy-as-code engines).

\paragraph{Contributions.} The contributions are:
\begin{enumerate}
\item \textbf{Composition closure as a formal authorization primitive.} A session-wide constraint on which action-type combinations are prohibited, enforced by infrastructure outside the model runtime. This primitive is not present as a constraint over action-type combinations in classical authorization models (RBAC, ABAC, OAuth) or, to our knowledge, in current agentic security systems. \textbf{Composition Soundness} (Theorem~\ref{thm:composition}) proves that the primitive is sound: no admissible action sequence produces a prohibited outcome given a complete effective restriction set $X_{\mathrm{eff}}$. $k$-tuple extensions (Proposition~\ref{prop:ktuple}) catch multi-step exfiltration that evades pairwise restrictions.
\item \textbf{Blast Radius Monotonicity across delegation chains.} A structural property (Theorem~\ref{thm:blast-radius}) establishing that the reachable blast radius is non-increasing at each delegation hop, with an explicit adversary model (\S\ref{sec:threat}) and a reference implementation tested in 99 delegation-chain scenarios at depths 2--8.
\item \textbf{Compromised-model evaluation methodology.} An evaluation design that tests infrastructure-level enforcement independently of whether the model resists manipulation, using ground-truth attack tool calls injected directly into the agent pipeline (Section~\ref{sec:eval}).
\item \textbf{Reference implementation} (Python, 2,500 LOC source, 3,000 LOC tests) with executable tests aligned to all formal properties, observed sub-millisecond enforcement overhead, and deployment mappings to enterprise infrastructure.
\end{enumerate}

The formal results are properties of the delegation algebra rather than of any given enforcement engine: they hold for any implementation faithful to APC semantics, given the stated assumptions.

The remainder of the paper is organized as follows. Section~\ref{sec:background} provides background on agentic tool use, delegated authority, and the limitations of static authorization. Section~\ref{sec:threat} specifies the threat model and security requirements. Section~\ref{sec:model} presents the session-scoped authorization model. Section~\ref{sec:enforcement} describes the runtime enforcement architecture. Section~\ref{sec:eval} reports the evaluation. Section~\ref{sec:discussion} discusses guarantees, deployment, and limitations; Section~\ref{sec:related} surveys related work; and Section~\ref{sec:conclusion} concludes.

\section{Background and Motivation}\label{sec:background}

\subsection{LLM-Based Agents and Tool Use}

An LLM-based agent is a system in which a language model is enabled to take actions in the world by emitting structured tool invocations. A tool invocation names an operation (e.g., \texttt{read}, \texttt{send\_email}, \texttt{transfer\_funds}) and a set of parameters; an execution layer dispatches the call to a backend---a local function, a REST API, a cloud service, or a Model Context Protocol (MCP) server. Agents operate in a loop: the model observes context, proposes an action, the action executes, and its result is appended to context for the next step. Workflows are therefore multi-step, stateful, and determined at runtime rather than statically specified.

Two properties of this execution model are security-relevant. First, the sequence of tool calls is not fixed in advance; it emerges from model inference over context that includes untrusted data (retrieved documents, tool outputs, sub-agent messages). Second, agents commonly \emph{delegate}: an orchestrator agent decomposes a task and assigns subtasks to qualified sub-agents, which may in turn delegate further. Authority must therefore flow across a chain of non-human principals, each of which is an attack surface.

\subsection{Delegated Authority and Non-Human Identities}

In a delegated workflow, a human principal initiates a task and authorizes an agent to act on their behalf. The agent is a \emph{non-human identity}: a principal with credentials and permissions but no independent intent. Established delegation mechanisms---OAuth 2.0 authorization grants, on-behalf-of (OBO) token exchange~\cite{rfc8693}, and Rich Authorization Requests~\cite{rfc9396}---answer the question ``may this client act for this user against this resource?'' at the moment a token is issued. Token Exchange does represent multi-party delegation: nested \texttt{act} claims record a chain of acting parties, and \texttt{may\_act} constrains who may act on behalf of whom. What these mechanisms do not provide is per-hop scope attenuation computed by infrastructure, constraints on which \emph{combinations} of authorized operations may be exercised, or authorization state that depends on what the delegate has already done.

Agentic workflows stress these mechanisms in two ways. First, recording a delegation chain is not the same as attenuating it: authority passes from user to orchestrator to sub-agent to tool, and each hop is a place where scope should \emph{narrow} and where compromise may occur, but the grant itself carries no per-hop reduction that infrastructure computes and enforces. Second, authority is exercised over an extended, stateful session in which the safety of an action depends on what has already happened, not only on a static grant. A token that authorizes ``read documents'' and ``send email'' does not encode the constraint that these two capabilities must not be combined to move confidential data to an external recipient.

\subsection{Why Static Authorization Fails in Stateful Agent Workflows}

Classical access control evaluates each request in isolation against a policy that is fixed for the duration of a session. RBAC~\cite{sandhu1996} and ABAC~\cite{hu2014} decide whether a subject may act on a resource based on roles or attributes; per-tool permissions and allowlists decide whether a given tool may be invoked. None of these mechanisms reason about \emph{sequences} of actions or about how the safety of one action depends on prior actions in the same session.

A compromised or manipulated agent can exploit this gap in three concrete ways while remaining within its nominal permissions:
\begin{itemize}
\item \textbf{Intent violation.} The agent performs actions that are individually permitted but unrelated to, or contrary to, the task the user delegated.
\item \textbf{Delegated-scope expansion.} A sub-agent attempts to exercise authority beyond what was passed to it, or authority propagates to additional sub-agents without attenuation.
\item \textbf{Unsafe composition (privilege propagation by combination).} The agent combines individually authorized actions---read a confidential document, then send an external message---to realize a prohibited outcome that no single permission forbids.
\end{itemize}

These are authorization failures, not model-alignment failures: they remain possible no matter how robust the model is to a given injection, because the underlying authority structure permits them. Addressing them requires authorization state that is \emph{scoped to the session}, \emph{narrowed across delegation}, and \emph{evaluated against prior-action state}.

\section{Threat Model and Security Requirements}\label{sec:threat}

We assume a modern enterprise agentic system: a human initiates a task; an orchestrator delegates to one or more sub-agents; tools are exposed via MCP servers or APIs; and policy infrastructure exists outside the model runtime.

\subsection{Assets}

The protected assets are: (i)~enterprise data, including confidential documents and customer records; (ii)~tool access and the operations tools expose; (iii)~cloud resources and the actions they permit; (iv)~external communication channels (email, messaging, web posting) that can serve as exfiltration sinks; (v)~financial and operational actions with irreversible effects (transfers, deletions, account changes); and (vi)~the audit trail itself, whose integrity is required to preserve accountability.

\subsection{Adversary Capabilities}

Table~\ref{tab:adversary} enumerates the adversary's capabilities. The adversary may inject content into the agent's context through untrusted data (indirect prompt injection); may fully compromise a single principal in the chain (a sub-agent, a tool server, or the orchestrator); may observe which actions succeed or fail to probe scope boundaries; and may maintain influence for the duration of a task session. We additionally consider a malicious low-privilege user or sub-agent attempting to exceed its delegated authorization. The baseline assumption is \emph{single-principal compromise}; the effect of \emph{multi-principal compromise} is discussed in Section~\ref{sec:discussion}.

\begin{table}[t]
\centering
\footnotesize
\caption{Adversary capabilities.}
\label{tab:adversary}
\begin{tabularx}{\linewidth}{@{}lX@{}}
\toprule
\textbf{ID} & \textbf{Adversary capability} \\
\midrule
A1 & Inject content into the agent's context via untrusted data sources (indirect prompt injection) \\
A2 & Fully compromise a single principal in the chain (sub-agent, tool server, or orchestrator) \\
A3 & Observe which actions succeed or fail to probe scope boundaries \\
A4 & Maintain influence for the duration of a task session \\
\bottomrule
\end{tabularx}
\end{table}

\subsection{Trusted and Untrusted Components}

The \textbf{trusted computing base} comprises the identity provider (IdP), the policy decision point (PDP), the policy enforcement gateway (PEP), the signed policy configuration, the approval service, and the append-only audit/evidence store. The adversary cannot compromise this infrastructure, forge cryptographic signatures or hashes, or operate across session boundaries (Table~\ref{tab:trust}).

The \textbf{untrusted components} are precisely those an attacker can influence: model outputs, retrieved content, user-provided documents, tool outputs, sub-agent messages, and natural-language instructions of any origin. All such content is treated as data, never as authorization; authorization derives solely from the signed session state held by infrastructure.

\begin{table}[t]
\centering
\footnotesize
\caption{Trust boundaries (capabilities the adversary does \emph{not} have).}
\label{tab:trust}
\begin{tabularx}{\linewidth}{@{}lX@{}}
\toprule
\textbf{ID} & \textbf{Trust boundary} \\
\midrule
T1 & Compromise the PDP, PEP, evidence sink, or key-management infrastructure \\
T2 & Forge cryptographic signatures or hashes \\
T3 & Operate across session boundaries \\
\bottomrule
\end{tabularx}
\end{table}

\subsection{Security Goals}

Under these conditions, the model targets six security goals (Table~\ref{tab:guarantees-main}).

\begin{table}[t]
\centering
\footnotesize
\caption{Security goals and the mechanisms that enforce them.}
\label{tab:guarantees-main}
\begin{tabularx}{\linewidth}{@{}l>{\raggedright\arraybackslash}p{0.46\linewidth}X@{}}
\toprule
\textbf{ID} & \textbf{Guarantee} & \textbf{Enforced by} \\
\midrule
G1 & No action outside scope $S(p_i)$ executes & Condition 2a \\
G2 & No prohibited action pair co-occurs in a session; $k$-tuple restrictions extend this to ordered subsequences & Condition 2b, Thm.~\ref{thm:composition}, Prop.~\ref{prop:ktuple} \\
G3 & Blast radius does not increase at each delegation hop & Thm.~\ref{thm:blast-radius} \\
G4 & High-impact actions require valid approval tokens & Condition 4 \\
G5 & Every admitted action is coupled to infrastructure-generated evidence, or execution is denied & Condition 5 \\
G6 & Actions outside declared intent are denied or flagged & Condition 6 \\
\bottomrule
\end{tabularx}
\end{table}

\subsection{Non-Goals}

The model does not attempt to: solve model alignment or prevent all hallucination; replace application-layer or backend authorization (it composes with, and does not supersede, the authorization checks performed by tools and services); prevent compromise of the trusted infrastructure itself (T1--T2); perform semantic parameter validation of every individual tool call; or prevent a compromised agent from choosing a suboptimal action \emph{within} its authorized scope together with intent. The last item constitutes the alignment boundary: APC bounds \emph{what} authority can be exercised and \emph{in what combinations}, not whether a single authorized action is the right one.

\section{Session-Scoped Authorization Model}\label{sec:model}

\subsection{Core Abstractions}\label{sec:core-abstractions}

The authorization subject in APC is not a single entity, but rather a \emph{chain of principals} bound to a session. A human principal $p_0$ initiates a task; an orchestrator $p_1$ acts on the human's behalf; sub-agents $p_2, \dots, p_n$ receive attenuated authority in turn; and tools execute at the end of the chain. Each principal is a distinct, verifiable identity (human, agent, or infrastructure), and authorization is decided for the \emph{acting principal at its position in the chain}, not for the human or the agent in isolation.

APC carries a \emph{session-level authorization state} that travels with this chain. The state is created at session initialization as a cryptographically signed \emph{authorization envelope}; the authority conferred by the envelope is narrowed---never widened---at each delegation hop. It comprises four elements: an authorization scope (\S\ref{sec:scope}), the principal chain and its delegation budget (\S\ref{sec:delegation}), the prior-action state of the session (\S\ref{sec:prior-action}), and a pre-declared intent specification (\S\ref{sec:intent}). APC is not an identity provider, and it does not replace backend authorization. It is an authorization-state layer that constrains delegated tool use and composes with existing identity and policy infrastructure.

\subsection{Authorization Scope}\label{sec:scope}

\begin{definition}[Authorization Scope]\label{def:scope}
An authorization scope $S = (R, A, D, X)$ consists of four components: $R$ (permitted resources), $A$ (permitted action types), $D$ (permitted data classifications), and $X \subseteq \binom{A}{2}$ (prohibited action compositions; the reference implementation extends this to ordered $k$-tuples). When two scopes are combined through delegation, the result is always narrower:
\[
S_1 \sqcap S_2 = (R_1 \cap R_2,\; A_1 \cap A_2,\; D_1 \cap D_2,\; X_1 \cup X_2).
\]
After the meet, $X$ may contain pairs referencing action types no longer in $A_1 \cap A_2$; such pairs are vacuously satisfied. The meet is associative, commutative, and idempotent. Restrictions can only accumulate; a downstream principal cannot remove them.
\end{definition}

For a scope $S = (R, A, D, X)$ we write $R(S)$, $A(S)$, $D(S)$, and $X(S)$ for its four projections. The properties proved here hold at the action-type level of abstraction: the set $A$ and the mapping $\mu$ from tools to action types determine the granularity at which composition restrictions operate. Finer taxonomies yield tighter enforcement; coarser taxonomies increase the risk of intent overlap (\S\ref{sec:eval}).

\subsection{Delegation Chains and Budgets}\label{sec:delegation}

Authority flows through a chain: user to orchestrator to sub-agent to tool. At each hop, the scope narrows and quantitative limits apply.

\begin{definition}[Principal Chain]\label{def:chain}
$C = \langle p_0, p_1, \ldots, p_n \rangle$ is the ordered sequence of principals, where $p_0$ is the human and each $p_i$ received authority from $p_{i-1}$.
\end{definition}

The Agentic Principal Chain is the session-level authorization state carried along $C$: the per-principal scope $S(p_i)$, the delegation budget, and the accumulated prior-action state. Infrastructure maintains it, signs it at creation, and consults it on every proposed action.

\begin{definition}[Delegation Budget]\label{def:budget}
\[
B = (\delta_{\max},\; \beta_{\max},\; \rho_{\max},\; \sigma_{\max},\; \kappa,\; \mathrm{cost}_{\max})
\]
Six ceilings: delegation depth, cumulative blast radius, irreversible effects, sensitivity class, cross-domain composition, and compute cost. Set at session initialization, non-negotiable by the agent, tracked by infrastructure. Budget ceilings can only decrease along the chain.
\end{definition}

\begin{definition}[Scoped Principal]\label{def:scoped}
At each hop, the scope narrows by the meet: $S(p_i) = S(p_{i-1}) \sqcap S_{\mathrm{role}}(p_i)$. This is computed by infrastructure, not by the agent.
\end{definition}

\subsection{Blast Radius and Containment}

\begin{definition}[Maximal Blast Radius]\label{def:blast}
\begin{equation*}
BR_{\max}(p_i) = R(S(p_i)) \cap \bigl\{r : \mathrm{blast}(r) \leq \beta_{\max}(p_i) - \beta_{\mathrm{consumed}}(p_i)\bigr\}
\end{equation*}
where $\mathrm{blast}(r) \in [0,1]$ is the normalized blast-radius contribution of resource~$r$, assigned by infrastructure at deployment time and consistent across all principals in the chain; $\beta_{\max}(p_i)$ is the blast-radius ceiling; and $\beta_{\mathrm{consumed}}(p_i)$ is the cumulative budget consumed up to~$p_i$. Budget consumption is updated atomically before each action executes. The method for deriving $\mathrm{blast}(r)$ and $\beta_{\max}$ from enterprise risk artifacts is given in Appendix~\ref{app:blast-calibration}.
\end{definition}

\begin{theorem}[Blast Radius Monotonicity]\label{thm:blast-radius}
For any principal chain: $\forall i:\; BR_{\max}(p_i) \subseteq BR_{\max}(p_{i-1})$.
\end{theorem}

\begin{proof}
By induction. \emph{Base:} scope narrowing gives $R(S(p_1)) \subseteq R(S(p_0))$. Budget-ceiling monotonicity (Definition~\ref{def:budget}) gives $\beta_{\max}(p_1) \leq \beta_{\max}(p_0)$. Budget consumption is cumulative: $p_1$ inherits the consumed budget of $p_0$ and can only increase it, so $\beta_{\mathrm{consumed}}(p_1) \geq \beta_{\mathrm{consumed}}(p_0)$. Hence
\[
\beta_{\max}(p_1) - \beta_{\mathrm{consumed}}(p_1) \leq \beta_{\max}(p_0) - \beta_{\mathrm{consumed}}(p_0).
\]
Both components of the intersection are subsets of their parent counterparts---scope by narrowing, and the budget set because $\{r : \mathrm{blast}(r) \leq a\} \subseteq \{r : \mathrm{blast}(r) \leq b\}$ whenever $a \leq b$---so $BR_{\max}(p_1) \subseteq BR_{\max}(p_0)$. \emph{Inductive step:} the identical argument applies at each subsequent hop.
\end{proof}

\noindent\emph{Structural requirement.} The proof depends on cumulative budget accounting along the chain. Each child principal must inherit the consumed budget of its parent as a floor, $\beta_{\mathrm{consumed}}(p_i) \geq \beta_{\mathrm{consumed}}(p_{i-1})$. Scope narrowing alone guarantees $R(S(p_i)) \subseteq R(S(p_{i-1}))$; the full $BR_{\max}$ monotonicity requires the cumulative-tracking invariant. The practical consequence is a \emph{containment guarantee}: if an attacker compromises a sub-agent at depth $k$, the damage is bounded by the scope and budget at that position.

\subsection{Session Intent}\label{sec:intent}

Classical authorization answers \emph{is this principal permitted to do this?} Intent binding answers an orthogonal question: \emph{is this action relevant to the declared task?} The intent specification $\Psi$ is \emph{pre-declared} by the session initiator at envelope creation; it is not inferred from the agent's behavior or derived by the model at runtime. $\Psi$ contains (i)~a task-objective string; (ii)~permitted resource patterns (glob-matched against the target resource); (iii)~permitted action sequences (matched against the action type); (iv)~negative constraints (resources explicitly prohibited regardless of other permissions); and optionally (v)~a fine-grained \emph{action--resource map} specifying, per action type, which resource patterns are permitted. Negative constraints are evaluated first and override all other permissions; where the action--resource map is present it takes precedence over coarse-grained patterns for mapped actions. The precondition $R_\Psi \subseteq R(S(p_i))$ is enforced at envelope creation, ensuring intent cannot widen scope.

\begin{proposition}[Intent Refinement]\label{prop:intent}
Let $R_\Psi$ and $A_\Psi$ denote the resources and action types permitted by $\Psi$. If $R_\Psi \subseteq R(S(p_i))$ and $A_\Psi \subseteq A(S(p_i))$, then for every action $a$, $\mathrm{Admit}_\Psi(a) \Rightarrow \mathrm{Admit}_S(a)$: intent only restricts, never widens.
\end{proposition}

\begin{proof}
Every resource pattern in $\Psi$ matches a subset of resources in $S(p_i)$, and every permitted action type in $\Psi$ is in $A(S(p_i))$ (validated at envelope creation). Any action admitted under $\Psi$ therefore satisfies both intent and scope; the converse does not hold.
\end{proof}

Intent binding supports graduated enforcement: \emph{strict} (deny), \emph{warn} (admit with elevated logging), and \emph{audit} (admit with a deviation flag). The envelope declares the mode, and the agent cannot modify it.

\subsection{Prior-Action State}\label{sec:prior-action}

Composition constraints are evaluated against the \emph{prior-action state} of the session: a running set of exercised action types together with an ordered history of action types. The PEP maintains this state incrementally as actions are admitted. Pairwise restrictions are checked in time linear in the number of distinct exercised types; ordered $k$-tuple restrictions are checked by subsequence matching over the history. The prior-action state is per-session and held by the infrastructure, not by the agent, so it cannot be reset or forged by model output.

\subsection{Composition Constraints}\label{sec:composition}

The component $X$ of a scope encodes \emph{composition closure}: the set of action-type combinations prohibited within a session. Pairwise restrictions $X \subseteq \binom{A}{2}$ prohibit co-occurrence of two action types; the implementation widens this to ordered $k$-tuples $K$ for attacks that require three or more action types where no individual pair is prohibited.

\begin{theorem}[Composition Soundness]\label{thm:composition}
Let a \emph{prohibited outcome} be any system effect that the security policy designates as impermissible (e.g., data exfiltration, unauthorized privilege escalation). Assume each prohibited outcome requires the co-occurrence of at least one specific action-type pair to be realized (single-action prohibitions are handled by scope attenuation, Condition~2a). If the composition restriction set $X$ covers at least one such required pair for every prohibited outcome, then no sequence of individually admissible actions can produce any prohibited outcome.
\end{theorem}

\begin{proof}
By contradiction. Assume sequential action evaluation (concurrent execution requires serialized admission). If an admissible sequence produces a prohibited outcome, some required pair $(t_i, t_j) \in X$ must have been exercised. But when $a_j$ was evaluated, $t_i$ was already in the exercised set, so the composition check would have denied $a_j$, contradicting admissibility.\qedhere
\end{proof}

\emph{Scope and limitations.} The theorem gives a conditional correctness result: the enforcement mechanism is sound relative to the policy specification $X$. The guarantee is conditional on three assumptions: (1)~completeness of $X$, (2)~serialized admission (concurrent sub-agents require the PEP to serialize composition checks so TOCTOU violations do not arise), and (3)~static $X$ within a session. The intent specification $\Psi$ can override specific pairs, yielding $X_{\mathrm{eff}} = X \setminus O(\Psi)$, where $O(\Psi)$ is the set of pairs for which the declared intent explicitly authorizes both action types. Where intent-derived overrides are used, $O(\Psi)$ is fixed before the session starts and unavailable to model output; the theorem applies to the resulting static $X_{\mathrm{eff}}$.

\begin{proposition}[Ordered $k$-tuple soundness]\label{prop:ktuple}
Let $K$ be a set of prohibited ordered tuples of action types. Assume the enforcement mechanism denies any candidate action whose admission would cause some tuple in $K$ to appear as an ordered subsequence of the session history. Then no admissible action sequence can realize any prohibited tuple in $K$.
\end{proposition}

\begin{proof}
By contradiction. Suppose an admissible sequence realizes some prohibited tuple $\tau = (t_1,\dots,t_k) \in K$. Let $a_j$ be the first action whose admission completes the final element $t_k$ of $\tau$ as an ordered subsequence. At the time $a_j$ is evaluated, the prefix $(t_1,\dots,t_{k-1})$ is already present in order, so admitting $a_j$ would cause $\tau$ to appear as a prohibited ordered subsequence. By assumption, the enforcement mechanism denies $a_j$, contradicting admissibility.
\end{proof}

\noindent Unlike pairwise restrictions, $k$-tuple restrictions are not subject to intent-derived overrides. The empirical cost of incomplete $X$ is quantified in Section~\ref{sec:injecagent}: removing a single pair raises data-stealing ASR from 0\% to 39.9\%.

\subsection{Policy Invariants}\label{sec:invariants}

The model is governed by a small set of invariants which any faithful implementation must preserve:
\begin{itemize}
\item \textbf{Monotone narrowing.} Across delegation, $R$, $A$, and $D$ can only shrink and $X$ can only grow (Definition~\ref{def:scope}); budget ceilings can only decrease (Definition~\ref{def:budget}). Narrowing is irreversible within a session.
\item \textbf{Cumulative consumption.} Consumed budget is inherited as a floor across hops, which is the structural requirement behind Theorem~\ref{thm:blast-radius}.
\item \textbf{Intent containment.} $R_\Psi \subseteq R(S(p_i))$ and $A_\Psi \subseteq A(S(p_i))$ at envelope creation, so intent restricts but never widens (Proposition~\ref{prop:intent}).
\item \textbf{Fail closed.} Missing, ambiguous, or unverifiable authority results in denial rather than execution.
\end{itemize}
Infrastructure enforces the invariants, which model output cannot modify.

\section{Runtime Enforcement Architecture}\label{sec:enforcement}

\subsection{PEP/PDP Architecture}\label{sec:pep-pdp}

Enforcement follows the standard separation between a policy decision point (PDP) and a policy enforcement point (PEP), with the PEP placed outside the model runtime. The PDP evaluates the admissibility predicate for each proposed action against the session's authorization state; the PEP gates execution on the PDP's decision and commits evidence. Figure~\ref{fig:architecture} shows the arrangement.

\begin{figure}[t]
\centering
\begin{tikzpicture}[
  node distance=0.5cm,
  principal/.style={rectangle, draw=black!50, fill=blue!8,
    minimum width=2.4cm, minimum height=0.5cm, font=\scriptsize, inner sep=2pt},
  infra/.style={rectangle, draw=black!50, fill=orange!10,
    minimum width=2.1cm, minimum height=0.5cm, font=\scriptsize, inner sep=2pt},
  store/.style={rectangle, draw=black!50, fill=green!8,
    minimum width=2.1cm, minimum height=0.5cm, font=\scriptsize, inner sep=2pt},
  toolbox/.style={rectangle, draw=black!50, fill=black!4,
    minimum width=2.4cm, minimum height=0.5cm, font=\scriptsize, inner sep=2pt},
  badge/.style={rectangle, draw=black!35, fill=orange!12, rounded corners=2pt,
    font=\scriptsize, inner sep=1.5pt},
  lbl/.style={font=\scriptsize, text=black!60},
  arr/.style={-{Stealth[length=1.4mm]}, semithick, black!65},
]
\node[principal] (p0) {Human Principal $p_0$};
\node[principal, below=of p0] (p1) {Orchestrator $p_1$};
\node[principal, below=of p1] (p2) {Sub-Agent $p_2$};
\node[toolbox, below=1.4cm of p2] (tool) {Tool Execution};
\draw[arr] (p0) -- (p1);
\draw[arr] (p1) -- (p2);
\node[lbl, right=2pt] at ($(p0.south)!0.5!(p1.north)$) {$\sqcap$ narrow};
\node[lbl, right=2pt] at ($(p1.south)!0.5!(p2.north)$) {$\sqcap$ narrow};
\draw[decorate, decoration={brace, amplitude=3.5pt}, semithick, black!30]
  ([xshift=-0.2cm]p0.north west) -- ([xshift=-0.2cm]p2.south west)
  node[midway, left=5pt, font=\scriptsize, text=black!60, align=right] {scope\\narrows};
\node[infra, right=1.6cm of p2] (pdp) {Policy Decision Pt.};
\node[infra, below=of pdp] (pep) {Enforcement Point};
\node[store, below=of pep] (ev) {Evidence Store};
\node[badge] at (pdp.north) [anchor=south] {C1 C2a C2b C2c C3 C4 C5 C6};
\draw[arr] (p2.east) -- ++(0.3,0) |- (pdp.west);
\draw[arr] (pdp) -- node[right, lbl] {admit/deny} (pep);
\draw[arr] (pep.west) -- ++(-0.3,0) |- (tool.east);
\draw[arr] (pep) -- node[right, lbl] {commit} (ev);
\node[font=\scriptsize, text=black!60, below=0.3cm of tool, align=center]
  {$BR_{\max}(p_2) \subseteq BR_{\max}(p_1) \subseteq BR_{\max}(p_0)$\\[1pt]
   \textnormal{Blast radius non-increasing (Thm.~\ref{thm:blast-radius})}};
\end{tikzpicture}
\caption{Delegation chain with scope narrowing ($\sqcap$) at each hop (blue). Infrastructure enforcement (orange): the PDP evaluates six admissibility conditions before the PEP permits execution. The PEP commits tamper-evident evidence to an append-only store; if the store is unreachable, the action is denied (C5).}
\Description{Architecture diagram showing a delegation chain from Human Principal p0 through Orchestrator p1 to Sub-Agent p2, with scope restricting at each hop. The sub-agent connects to a Policy Decision Point that evaluates six admissibility conditions before the Enforcement Point permits tool execution.}
\label{fig:architecture}
\end{figure}
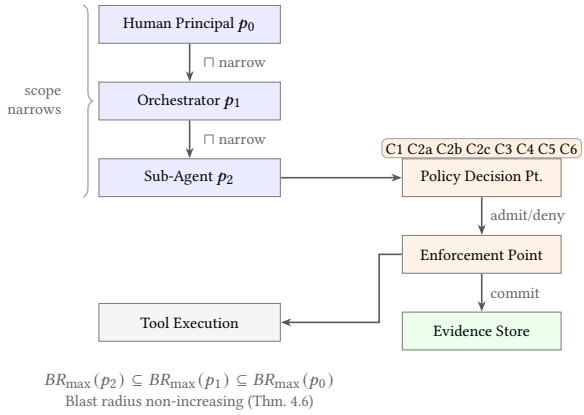

\subsection{Tool Gateway Integration}\label{sec:gateway}

The PEP is implemented as a tool gateway or an MCP gateway, sitting between the agent runtime and the backends it calls. Every tool invocation---whether a direct API call, a cloud-service action, or an MCP server request---passes through this gateway, which holds the signed authorization envelope for the session and narrows it at each delegation hop. Because the gateway is the single chokepoint for action execution, the model cannot bypass it by emitting alternative text: an action not admitted at the gateway never reaches the backend. The gateway does not replace the backend's own authorization; it is an additional, session-scoped enforcement layer that fails closed.

\subsection{Policy Evaluation Lifecycle}\label{sec:lifecycle}

Each time the agent proposes an action, the PDP evaluates a conjunctive predicate over six conditions. If any condition does not hold, the action is denied:
\begin{multline*}
\mathrm{Admissible}(a, C, S, B, \mathcal{A}, E, \Psi) = \mathrm{true} \\
\iff C_1 \wedge C_2 \wedge C_3 \wedge C_4 \wedge C_5 \wedge C_6
\end{multline*}
where $a$ is the proposed action, $C$ the principal chain, $S$ the effective scope, $B$ the budget state, $\mathcal{A}$ the approval store, $E$ the evidence-sink state, and $\Psi$ the intent specification.

\textbf{C1: Identity Binding.} The runtime actor must be bound to a verifiable identity in the principal chain, distinguishing user from agent from infrastructure.

\textbf{C2: Scope Attenuation with Composition Closure.} Three subchecks: \emph{(2a)}~the action is within the attenuated scope---action type in $A(S)$, target resource in $R(S)$, data classification in $D(S)$. \emph{(2b)}~The action does not create a prohibited combination with previous actions in the session. \emph{(2c)}~The action satisfies all delegation-budget ceilings.

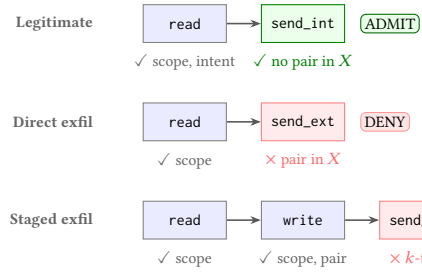
\begin{figure}[t]
\centering
\begin{tikzpicture}[
  node distance=0.4cm and 0.45cm,
  action/.style={rectangle, draw=black!50, fill=blue!8,
    minimum width=1.1cm, minimum height=0.5cm, font=\scriptsize, inner sep=2pt},
  ok/.style={rectangle, draw=green!50!black, fill=green!8,
    minimum width=1.1cm, minimum height=0.5cm, font=\scriptsize, inner sep=2pt},
  blocked/.style={rectangle, draw=red!50, fill=red!8,
    minimum width=1.1cm, minimum height=0.5cm, font=\scriptsize, inner sep=2pt},
  badge-ok/.style={rectangle, draw=green!40!black, fill=green!10, rounded corners=2pt,
    font=\scriptsize, inner sep=1.5pt},
  badge-deny/.style={rectangle, draw=red!40, fill=red!10, rounded corners=2pt,
    font=\scriptsize, inner sep=1.5pt},
  lbl/.style={font=\scriptsize, text=black!60},
  arr/.style={-{Stealth[length=1.4mm]}, semithick, black!65},
]
\node[lbl, anchor=east, font=\scriptsize\bfseries] at (-0.1,0) {Legitimate};
\node[action] (r1) at (1.0,0) {\texttt{read}};
\node[ok, right=of r1] (s1) {\texttt{send\_int}};
\node[badge-ok, right=0.2cm of s1] {ADMIT};
\draw[arr] (r1) -- (s1);
\node[lbl, below=1pt of r1] {$\checkmark$ scope, intent};
\node[font=\scriptsize, text=green!50!black, below=1pt of s1] {$\checkmark$ no pair in $X$};
\node[lbl, anchor=east, font=\scriptsize\bfseries] at (-0.1,-1.3) {Direct exfil};
\node[action] (r2) at (1.0,-1.3) {\texttt{read}};
\node[blocked, right=of r2] (s2) {\texttt{send\_ext}};
\node[badge-deny, right=0.2cm of s2] {DENY};
\draw[arr] (r2) -- (s2);
\node[lbl, below=1pt of r2] {$\checkmark$ scope};
\node[font=\scriptsize, text=red!60, below=1pt of s2] {$\times$ pair in $X$};
\node[lbl, anchor=east, font=\scriptsize\bfseries] at (-0.1,-2.6) {Staged exfil};
\node[action] (r3) at (1.0,-2.6) {\texttt{read}};
\node[action, right=of r3] (w3) {\texttt{write}};
\node[blocked, right=of w3] (s3) {\texttt{send\_int}};
\node[badge-deny, right=0.2cm of s3] {DENY};
\draw[arr] (r3) -- (w3);
\draw[arr] (w3) -- (s3);
\node[lbl, below=1pt of r3] {$\checkmark$ scope};
\node[lbl, below=1pt of w3] {$\checkmark$ scope, pair};
\node[font=\scriptsize, text=red!60, below=1pt of s3] {$\times$ $k$-tuple};
\end{tikzpicture}
\caption{Composition closure distinguishes legitimate workflows from attacks. \emph{Top:} reading then sharing internally is admitted. \emph{Middle:} direct exfiltration is blocked by pairwise restriction. \emph{Bottom:} staged exfiltration via an intermediate write evades pairwise but is caught by $k$-tuple restriction.}
\Description{Three-row diagram showing how composition closure works.}
\label{fig:composition}
\end{figure}

\textbf{C3: Context and State Binding.} The action must be bound to the specific task instance, policy version, and parameter context to prevent replay attacks across sessions.

\textbf{C4: Approval Binding.} The system computes an impact score $I(a) = w_\rho \cdot \rho(a) + w_\beta \cdot Bl(a) + w_\sigma \cdot Se(a)$, where $\rho(a) \in [0,1]$ is irreversibility, $Bl(a) \in [0,1]$ is blast radius, and $Se(a) \in [0,1]$ is data sensitivity. If $I(a) > \theta$, a single-use, hash-bound approval token is required---tied to the exact action, parameters, and session. Here $Bl(a)$ is the blast-radius contribution of the action itself (set per action profile at deployment time), distinct from the per-resource $\mathrm{blast}(r)$ used in budget accounting (C2c, Definition~\ref{def:blast}). Approval binding is per-action; budget conformance (C2c) is cumulative. The two are complementary, not redundant. Calibration is detailed in Appendix~\ref{app:calibration}.

\textbf{C5: Evidence Commitment.} Before admitting an action, the PDP verifies that the evidence sink is reachable. If unreachable, the action is denied---no evidence trail means no execution. A SHA-256 hash chain enforces tamper-evidence: each package records the content hash of the preceding entry, so modifying or removing any interior entry breaks the chain and is detectable by traversal. Truncation of the chain \emph{tail} is not detectable from the chain alone; detecting it requires an independently anchored head (e.g., periodic publication of the latest hash) or the append-only store assumed in the trusted computing base (T1). Production deployments should provide both.

\textbf{C6: Intent Binding.} The action is checked against the pre-declared intent specification $\Psi$. Negative constraints are evaluated first and override all other permissions. If no intent specification is available, admissibility falls back to C1--C5.

\paragraph{Guarantee tiers.} The six conditions fall into three tiers:
\begin{itemize}
\item \textbf{Tier~1 (structural):} C2a, C2b, C3---properties proved relative to a fixed effective policy. Theorem~\ref{thm:composition} and Proposition~\ref{prop:ktuple} operate at this tier.
\item \textbf{Tier~2 (configuration-dependent):} C4, C6---effectiveness depends on calibration and on the completeness of the intent specification. Residual ASR concentrates here.
\item \textbf{Tier~3 (operational/infrastructure):} C1, C5, C2c---depends on infrastructure availability and integrity.
\end{itemize}

\subsection{Audit and Trace Generation}\label{sec:audit}

Every admitted action is coupled to an infrastructure-generated evidence package committed to the append-only store before execution (C5). Each package records the action, principal, scope, and policy version, parameter context, decision, and the hash chain linking it to the prior entry. Because evidence commitment is a precondition for execution and the store is append-only, the audit trail is both complete (no executed action lacks evidence) and tamper-evident (no entry can be altered or removed undetected).

\subsection{Composition Restriction Authoring}\label{sec:restriction-authoring}

The restriction set $X$ is authored per security domain, not per use case, analogous to separation-of-duty constraints in IAM. In the evaluated domains, 3--9 pairwise restrictions and 0--8 $k$-tuple restrictions suffice (Table~\ref{tab:authoring}); the authoring follows a four-step procedure with a quantifiable coverage metric (Appendix~\ref{app:restriction-authoring}). The restriction set scales with prohibited outcomes, not with tool count.

\section{Evaluation}\label{sec:eval}

\begin{figure}[t]
\centering
\begin{tikzpicture}
\begin{axis}[
  ybar=2pt,
  width=0.95\linewidth,
  height=5.4cm,
  bar width=11pt,
  ymin=0, ymax=113,
  ylabel={Attack Success Rate (\%)},
  ylabel style={font=\footnotesize, at={(axis description cs:-0.08,0.5)}},
  ytick={0,25,50,75,100},
  yticklabel style={font=\footnotesize},
  ymajorgrids=true,
  grid style={gray!15, thin},
  axis line style={gray!60, thin},
  tick style={gray!60, thin},
  xtick=data,
  symbolic x coords={Exfiltration, Destruction, Manipulation},
  xticklabel style={font=\footnotesize},
  enlarge x limits=0.28,
  legend image code/.code={%
    \draw[#1, draw=none] (0cm,-0.1cm) rectangle (0.3cm,0.2cm);},
  legend style={at={(0.5,-0.16)}, anchor=north, font=\footnotesize,
    draw=none, column sep=8pt, legend columns=2},
  legend cell align=left,
  nodes near coords,
  nodes near coords style={font=\scriptsize, anchor=south, yshift=2pt},
  point meta=explicit symbolic,
  clip=false,
]
\addplot[fill=orange!35, draw=orange!50, draw opacity=0.6]
  coordinates {
    (Exfiltration, 87) [{87\%}]
    (Destruction, 39) [{39\%}]
    (Manipulation, 90) [{90\%}]
  };
\addplot[fill=black!65, draw=black!75, draw opacity=0.6]
  coordinates {
    (Exfiltration, 0.7) [{\textbf{0\%}}]
    (Destruction, 4) [{\textbf{4\%}}]
    (Manipulation, 12) [{\textbf{12\%}}]
  };
\legend{No defense, APC}
\end{axis}
\end{tikzpicture}
\caption{Compromised-model evaluation by attack type, pooled across all four AgentDojo domains (609 task--injection pairs). Exfiltration is eliminated by composition closure (C2b); destruction and manipulation are reduced by intent binding (C6).}
\Description{Vertical bar chart showing attack success rates by type. Exfiltration drops from 87\% to 0\%, destruction from 39\% to 4\%, manipulation from 90\% to 12\%.}
\label{fig:results-type}
\end{figure}
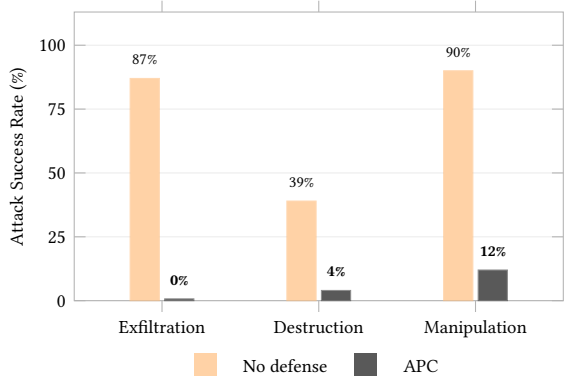

\begin{figure}[t]
\centering
\begin{tikzpicture}
\begin{axis}[
  ybar=2pt,
  width=0.95\linewidth,
  height=5.4cm,
  bar width=9pt,
  ymin=0, ymax=113,
  ylabel={Attack Success Rate (\%)},
  ylabel style={font=\footnotesize, at={(axis description cs:-0.08,0.5)}},
  ytick={0,25,50,75,100},
  yticklabel style={font=\footnotesize},
  ymajorgrids=true,
  grid style={gray!15, thin},
  axis line style={gray!60, thin},
  tick style={gray!60, thin},
  xtick=data,
  symbolic x coords={Data Stealing, Disruptive, Direct Harm, Stealthy},
  xticklabel style={font=\footnotesize},
  enlarge x limits=0.18,
  legend image code/.code={%
    \draw[#1, draw=none] (0cm,-0.1cm) rectangle (0.3cm,0.2cm);},
  legend style={at={(0.5,-0.16)}, anchor=north, font=\footnotesize,
    draw=none, column sep=8pt, legend columns=2},
  legend cell align=left,
  nodes near coords,
  nodes near coords style={font=\scriptsize, anchor=south, yshift=2pt},
  point meta=explicit symbolic,
  clip=false,
]
\addplot[fill=orange!35, draw=orange!50, draw opacity=0.6]
  coordinates {
    (Data Stealing, 100) [{100\%}]
    (Disruptive, 100) [{100\%}]
    (Direct Harm, 100) [{100\%}]
    (Stealthy, 100) [{100\%}]
  };
\addplot[fill=black!65, draw=black!75, draw opacity=0.6]
  coordinates {
    (Data Stealing, 0.7) [{\textbf{0\%}}]
    (Disruptive, 0.7) [{\textbf{0\%}}]
    (Direct Harm, 60) [{\textbf{60\%}}]
    (Stealthy, 30) [{\textbf{30\%}}]
  };
\legend{No defense, APC}
\end{axis}
\end{tikzpicture}
\caption{Static deterministic benchmarks: InjecAgent (1,054 cases) and ASB (400 cases). Data-stealing and disruptive attacks are blocked completely; residual direct-harm attacks (single-action within scope) and stealthy attacks (action-type granularity) lie outside APC's design boundary.}
\Description{Vertical bar chart showing static benchmark results. Data stealing and disruptive drop from 100\% to 0\%. Direct harm remains at 60\% and stealthy at 30\%.}
\label{fig:results-bench}
\end{figure}
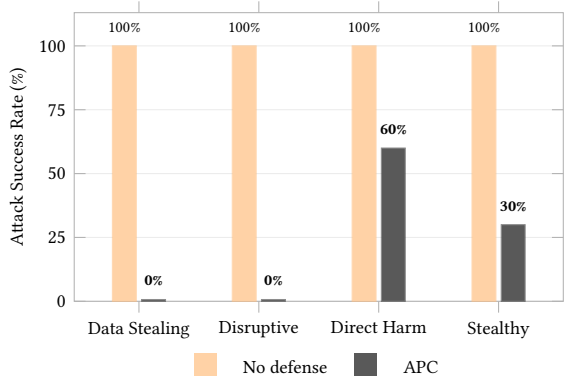

We evaluate APC across components with distinct evidentiary roles. \emph{Security evidence}: InjecAgent and ASB validate composition closure (C2b) deterministically across 1,454 cases; the compromised-model AgentDojo evaluation validates enforcement when the model emits the benchmark's ground-truth attack call across 609 task--injection pairs; the adaptive-attack suite validates all six conditions and both formal results under white-box adversarial assumptions. \emph{Utility evidence}: standard AgentDojo runs measure enforcement cost on legitimate task completion across 949 pairs. \emph{Formal validation}: 99 delegation-chain scenarios validate Theorem~\ref{thm:blast-radius} and the full six-condition predicate at depths 2--8. Reported security results are \emph{observed} attack effectiveness rates under the stated methodologies, not worst-case bounds. Figures~\ref{fig:results-type} and~\ref{fig:results-bench} summarize the headline security outcomes.

\subsection{Reference Implementation Validation}

The reference implementation (Python, 2,500 LOC source, 3,000 LOC tests) implements executable tests aligned with all formal properties, with noted sub-millisecond local enforcement overhead (full six-condition admissibility 0.05\,ms p50, 0.24\,ms p99, highest observed p99 across repetitions 0.35\,ms; environment, method, and per-component ranges in Appendix~\ref{app:implementation}). This measures only the APC authorization path---it excludes model inference, network round trips, and end-to-end agent runtime. Additionally, 99 delegation-chain scenarios validate the full six-condition predicate across 2--8 hop chains with 20 attack categories; all 99 produce the expected outcome (23 positive, 76 negative). Details in Appendix~\ref{app:implementation}.

\subsection{InjecAgent: Composition Closure}\label{sec:injecagent}

InjecAgent~\cite{zhan2024injecagent} provides 1,054 test cases covering 17 user tools and 62 attacker tools, with attacks split into \emph{data stealing} (544 cases) and \emph{direct harm} (510 cases). Evaluation is deterministic: InjecAgent pre-fills the agent scratchpad, and we evaluate whether the APC admissibility predicate blocks the attacker tool given the session's composition history. Table~\ref{tab:injecagent} shows the results.

\begin{table}[t]
\centering
\footnotesize
\caption{InjecAgent results (1,054 test cases). ASR = attack success rate (lower is better).}
\label{tab:injecagent}
\begin{tabularx}{\linewidth}{@{}l*{3}{>{\centering\arraybackslash}X}@{}}
\toprule
\textbf{Condition} & \textbf{Data Stealing ASR} & \textbf{Direct Harm ASR} & \textbf{Total ASR} \\
\midrule
Baseline & 100.0\% & 100.0\% & 100.0\% \\
Complete $X$ & \textbf{0.0\%} & 60.4\% & 29.2\% \\
Partial $X$ & 39.9\% & 82.0\% & 60.2\% \\
\bottomrule
\end{tabularx}
\end{table}

Complete $X$ reduces data stealing from 100\% to 0\%. Direct harm drops to 60.4\%: the remaining cases are single-action attacks within scope that require parameter-level validation. The Partial $X$ condition (one pair removed from $X$) raises data-stealing ASR to 39.9\%, illustrating the cost of violating the completeness assumption of Theorem~\ref{thm:composition}.

\subsection{ASB: Composition Closure Across Domains}

Agent Security Bench (ASB)~\cite{zhang2024asb} provides 400 attack tools across 10 agent scenarios, alongside 20 normal tools. We apply the same deterministic methodology. Table~\ref{tab:asb} shows the results.

\begin{table}[t]
\centering
\footnotesize
\caption{ASB results (400 attack tools, 10 scenarios). ASR = attack success rate.}
\label{tab:asb}
\begin{tabularx}{\linewidth}{@{}l*{3}{>{\centering\arraybackslash}X}@{}}
\toprule
\textbf{Condition} & \textbf{Stealthy ASR} & \textbf{Disruptive ASR} & \textbf{Total ASR} \\
\midrule
Baseline & 100.0\% & 100.0\% & 100.0\% \\
Complete $X$ & 30.0\% & \textbf{0.0\%} & \textbf{15.0\%} \\
Partial $X$ & 60.0\% & 100.0\% & 80.0\% \\
\bottomrule
\end{tabularx}
\end{table}

Complete $X$ blocks all 200 disruptive attacks and 140 of 200 stealthy attacks. Two surviving tool types (\texttt{DataExfiltrationTool}, \texttt{DataExportTool}) carry exfiltration semantics but receive the \texttt{write} class from the benchmark-level classifier, which keys on attack type rather than tool semantics. Under the correct class, the existing pair $(\texttt{read}, \texttt{send\_external}) \in X$ blocks both, confirming that this residual lies in the action-type taxonomy $\mu$ rather than in composition closure---the granularity dependence identified in Section~\ref{sec:scope}. An independent semantic audit of every benchmark residual, classified from the benchmarks' own natural-language attack descriptions rather than from our action-type labels, is provided in the artifact.

\subsection{AgentDojo: Utility and Live Enforcement}\label{sec:agentdojo}

AgentDojo~\cite{debenedetti2024} evaluates both attack success and task utility with live LLM agents. We integrate the APC enforcement layer as a pipeline element within AgentDojo's tool execution loop.

\paragraph{Model selection.} Utility preservation is evaluated with Claude Sonnet 4.6 (via AWS Bedrock). The compromised-model evaluation uses Claude Haiku 4.5 because the malicious tool call is injected directly from benchmark ground truth; the security outcome is driven primarily by infrastructure-level enforcement.

\paragraph{Per-domain configuration.} Each suite requires a domain-specific action classifier, composition restrictions, and intent-parser keywords. Table~\ref{tab:authoring} quantifies the per-domain effort.

\begin{table}[t]
\centering
\footnotesize
\caption{Per-domain authoring burden.}
\label{tab:authoring}
\begin{tabular}{@{}lrrrrr@{}}
\toprule
\textbf{Domain} & \textbf{Tools} & \textbf{Action Cls.} & \textbf{Pairwise} & \textbf{$k$-tuple} & \textbf{Resources} \\
\midrule
Workspace  & 32 & 8 & 7 & 8 & 5 \\
Banking    & 13 & 6 & 5 & 3 & 5 \\
Travel     & 28 & 7 & 3 & 4 & 7 \\
Slack      & 11 & 6 & 3 & 2 & 4 \\
\midrule
InjecAgent & --- & 9 & 8 & 0 & --- \\
ASB        & --- & 7 & 9 & 0 & --- \\
\bottomrule
\end{tabular}
\end{table}

\paragraph{Utility preservation.} We report UPR under two configurations: \emph{strict} (all conditions enforced, approval-gated actions denied) and \emph{interactive} (human approval simulated for C4-only denials). Results are means over 3 independent runs (Table~\ref{tab:utility}).

\begin{table}[t]
\centering
\scriptsize
\caption{Utility preservation on four AgentDojo suites (Claude Sonnet 4.6). UPR = fraction of task--injection pairs where the user task is completed successfully. Values are means over 3 runs; $\pm$ denotes sample standard deviation. $\Delta$ columns report percentage-point differences computed from unrounded means.}
\label{tab:utility}
\begin{tabularx}{\linewidth}{@{}l*{5}{>{\centering\arraybackslash}X}@{}}
\toprule
\textbf{Suite} & \textbf{Baseline} & \textbf{Strict} & \textbf{Inter.} & \textbf{$\Delta$ strict (pp)} & \textbf{$\Delta$ inter.\ (pp)} \\
\midrule
Workspace (560) & 96.2$\pm$0.2\% & 81.5$\pm$0.1\% & 86.7$\pm$0.4\% & $-$14.7 & $-$9.6 \\
Banking (144) & 68.8$\pm$1.2\% & 50.0$\pm$0.7\% & 58.8$\pm$1.7\% & $-$18.8 & $-$10.0 \\
Travel (140) & 59.8$\pm$1.1\% & 55.7$\pm$0.7\% & 60.0$\pm$1.4\% & $-$4.0 & $+$0.2 \\
Slack (105) & 73.0$\pm$1.5\% & 56.5$\pm$5.2\% & 60.0$\pm$4.4\% & $-$16.5 & $-$13.0 \\
\bottomrule
\end{tabularx}
\end{table}

\noindent Pair-weighted aggregate deltas across the 949 task--injection pairs are $-$13.9\,pp under the strict condition and $-$8.6\,pp under the interactive condition. The tradeoff is explicit: the $-$8.6\,pp interactive cost buys 0\% observed exfiltration ASR under full model compromise across all four domains.

\paragraph{Compromised-model evaluation.} Assuming the model may emit an attacker-chosen tool call is not hypothetical: the Gray Swan IPI Arena~\cite{grayswan2026} reports that all 13 frontier models tested were induced to comply at least some of the time (0.5--8.5\% ASR over 272{,}000 attacks), and the TRAP benchmark~\cite{trap2025} reports 13--43\% ASR across six frontier models. To evaluate APC independently of model strength, we instrument the AgentDojo pipeline (v1.2.2) with a compromised-model injection element that appends the ground-truth attack tool call after the first legitimate tool call, simulating full model compromise. The undefended baseline is not uniformly 100\% because AgentDojo's success checker evaluates the full task state: in some pairs the injected tool call executes, but its preconditions are not met (e.g., the target data was not yet in the agent's context), so the checker does not score the attack as successful. The cohort is smaller than the utility cohort (609 versus 949 pairs) because injection is defined only where the benchmark specifies one; we retain every user task, but only injection tasks whose AgentDojo ground truth contains at least one attack tool call. Injection tasks with empty ground truth provide nothing to inject and are excluded. Table~\ref{tab:compromised} shows results.

\begin{table}[t]
\centering
\footnotesize
\caption{Compromised-model evaluation on four AgentDojo suites (Claude Haiku 4.5). Ground-truth attack injection simulates a fully compromised model. Total: 609 unique task--injection pairs, 1,218 executions.}
\label{tab:compromised}
\begin{tabularx}{\linewidth}{@{}llcc>{\raggedright\arraybackslash}X@{}}
\toprule
\textbf{Suite / Category} & \textbf{Attack Tools} & \textbf{No Def.} & \textbf{APC} & \textbf{Blocking} \\
\midrule
\multicolumn{5}{@{}l}{\textit{Workspace (240 pairs)}} \\
\quad Exfiltration & send\_email & 90.0\% & \textbf{0.0\%} & C2b + C4 + C6 \\
\quad Destruction & delete\_file/email & 48.8\% & \textbf{5.0\%} & C6 (intent) \\
\quad Manipulation & create\_calendar & 97.5\% & 30.0\% & C6 (partial) \\
\midrule
\multicolumn{5}{@{}l}{\textit{Banking (144 pairs)}} \\
\quad Financial exfil. & send\_money & 75.0\% & \textbf{0.0\%} & C2b (formal) \\
\quad Manipulation & update\_sched\_txn & 87.5\% & 12.5\% & C6 (partial) \\
\quad Account takeover & update\_password & 87.5\% & \textbf{0.0\%} & C2b (formal) \\
\midrule
\multicolumn{5}{@{}l}{\textit{Travel (120 pairs)}} \\
\quad Exfiltration & send\_email & 90.0\% & \textbf{0.0\%} & C2b (formal) \\
\quad Manipulation & reserve/calendar & 86.7\% & \textbf{0.0\%} & C6 (no overlap) \\
\midrule
\multicolumn{5}{@{}l}{\textit{Slack (105 pairs)}} \\
\quad Ext.\ exfiltration & post\_webpage & 100.0\% & \textbf{0.0\%} & C2b (formal) \\
\quad Int.\ exfiltration & send\_direct\_msg & 100.0\% & \textbf{0.0\%} & $k$-tuple \\
\quad Destruction & remove\_user & 0.0\% & 0.0\% & --- (attack fails undefended) \\
\quad Reconnaissance & get\_webpage & 100.0\% & \textbf{0.0\%} & C6 (intent) \\
\bottomrule
\end{tabularx}
\end{table}

\noindent Exfiltration attacks are blocked at 0\% observed ASR in all four suites. Aggregated across suites, destruction ASR drops from 38.6\% to 4.0\%, and manipulation ASR drops from 90.5\% to 12.1\%.

\paragraph{Residual failure taxonomy.} Across all four domains, 18 attacks succeed despite enforcement (Table~\ref{tab:residual}). \emph{Intent overlap} (14 cases, 78\%): the attack action class is indistinguishable from the legitimate task at the action-type level. \emph{Missing composition pair} (4 cases, 22\%): a policy-completeness issue. Neither category is a core limitation.

\begin{table}[t]
\centering
\footnotesize
\caption{Taxonomy of residual attack successes (18 of 609 runs, 3.0\% aggregated observed ASR).}
\label{tab:residual}
\begin{tabularx}{\linewidth}{@{}ll r >{\raggedright\arraybackslash}X@{}}
\toprule
\textbf{Category} & \textbf{Attack Tool} & \textbf{Cases} & \textbf{Root Cause} \\
\midrule
\multirow{2}{*}{Intent overlap} & \texttt{create\_calendar\_event} & 12 & \texttt{write} class matches task intent \\
 & \texttt{update\_sched\_transaction} & 2 & \texttt{write} class matches task intent \\
\midrule
Missing pair in $X$ & \texttt{delete\_file} & 4 & (\texttt{read}, \texttt{delete}) $\notin X$ \\
\midrule
\multicolumn{2}{@{}l}{\textbf{Total}} & \textbf{18} & \\
\bottomrule
\end{tabularx}
\end{table}

\paragraph{Matched-protocol comparison.} In a matched-protocol comparison using Progent's~\cite{shi2025progent} AgentDojo fork adapted to Bedrock (artifact as available at evaluation time, manual policies, same model, two independent runs per suite), both systems achieved 0\% observed ASR across all available domains, and attack-time utility was comparable in banking and workspace: Progent $59.7\pm1.0$\% / $88.5\pm0.3$\%; APC interactive $58.8\pm1.7$\% / $86.7\pm0.4$\%. The weighted mean delta is $+1.6$\,pp in favor of Progent, within a predefined $\pm2$\,pp equivalence threshold. The APC figures in this comparison are those of Table~\ref{tab:utility} and are reproducible from the committed run files; the Progent-side figures were obtained by running that project's own fork and are reported here without a corresponding artifact in our repository, since redistributing it is outside our control.

\subsection{Adaptive Attacks}\label{sec:adaptive}

We design twenty-three attacks by an adversary with full knowledge of the model, covering the complete attack surface: all six conditions, both formal results, and all adversary capabilities (A1--A4) and trust boundaries (T1--T2). All twenty-three named attacks produce outcomes consistent with the model's predictions. Of the 43 variants, 19 are positive cases (admitted by design), and 24 are attack variants targeting prohibited outcomes. Of these, 23 are blocked; session splitting is admitted by design (per-session composition state is a documented limitation, T3). Key findings: decomposed exfiltration (\texttt{read} $\to$ \texttt{write} $\to$ \texttt{send\_internal}) evades pairwise closure but is caught by $k$-tuple restrictions; approval replay, expired tokens, and consumed tokens confirm C4 integrity; evidence evasion confirms fail-closed behavior (C5); session splitting confirms per-session composition state---cross-session attacks are admitted, a documented limitation. Representative results are in Appendix~\ref{app:adaptive}.

\subsection{Evaluation Summary}

Tables~\ref{tab:eval-summary} and~\ref{tab:attack-coverage} summarize coverage across all benchmarks and attack classes.

\begin{table}[t]
\centering
\footnotesize
\caption{Evaluation coverage across all benchmarks (3,154 evaluation instances). The compromised-model and utility cohorts are drawn from the same AgentDojo task--injection space and are therefore not disjoint.}
\label{tab:eval-summary}
\begin{tabularx}{\linewidth}{@{}lrlX@{}}
\toprule
\textbf{Benchmark} & \textbf{Cases} & \textbf{Type} & \textbf{Main outcome} \\
\midrule
Delegation chains & 99 & Multi-hop (2--8 hops) & 99/99, all 6 conditions \\
InjecAgent & 1,054 & Public & Data stealing 0\% \\
ASB & 400 & Public & Disruptive 0\% \\
AgentDojo (utility) & 949 & Live LLM, 4 suites & $\Delta$ interactive $-$8.6\,pp (mean, 3 runs) \\
AgentDojo (compromised) & 609$^\dagger$ & Compromised-model & Exfil 0\% all suites \\
Adaptive & 43 & Self-designed & 23/23 matched \\
\bottomrule
\multicolumn{4}{@{}l}{\scriptsize $^\dagger$609 unique pairs, 1,218 total executions under two conditions.}
\end{tabularx}
\end{table}

\begin{table}[t]
\centering
\footnotesize
\caption{Attack-class coverage summary. Residual values are observed ASR.}
\label{tab:attack-coverage}
\begin{tabularx}{\linewidth}{@{}ll>{\raggedright\arraybackslash}Xc@{}}
\toprule
\textbf{Attack class} & \textbf{Coverage} & \textbf{Main mechanism} & \textbf{Residual} \\
\midrule
Multi-step exfiltration & Full & C2b composition closure & 0\% \\
Delegation priv.\ escalation & Full & Thm.~\ref{thm:blast-radius} & 0\% \\
Destruction (out-of-intent) & High & C6 intent binding & 4\% \\
Manipulation (in-intent) & Partial & C6 (action-type granularity) & 12.1\% \\
Single-action within scope & None & Outside APC boundary & 60.4\% \\
Parameter-level misuse & None & Requires param.\ validation & --- \\
\bottomrule
\end{tabularx}
\end{table}

\section{Discussion}\label{sec:discussion}

\subsection{Relationship to RBAC, ABAC, OAuth, and OBO}\label{sec:relationship}

APC is designed to complement, not replace, established access-control and delegation mechanisms. RBAC~\cite{sandhu1996} and ABAC~\cite{hu2014} decide individual requests against roles or attributes; APC layers a session-scoped, sequence-aware constraint (composition closure) on top of these per-request decisions. OAuth 2.0 and Rich Authorization Requests~\cite{rfc9396} issue delegated grants, and Token Exchange~\cite{rfc8693} can represent a chain of acting parties using nested \texttt{act} claims; the Authorization Envelope is conceptually aligned with a rich authorization request but adds what those mechanisms leave to the application: scope attenuation computed per hop by infrastructure, prohibited action-type combinations, and admissibility conditioned on prior-action state. In a deployment, the IdP and OAuth/OBO flows establish identity and the initial grant; APC narrows that grant across the principal chain and enforces composition together with intent at the gateway. Backend services retain their own authorization; APC adds a session-scoped enforcement layer in front of them.

\subsection{Guarantees and Assumptions}\label{sec:guarantees}

The two formal results are properties of the delegation algebra, not of a particular checker. Theorem~\ref{thm:blast-radius} (blast-radius monotonicity) assumes scope narrowing, cumulative budget tracking, and consistent $\mathrm{blast}(r)$ assignment by infrastructure; under these, the reachable blast radius is non-increasing at each hop, so compromise at depth $k$ can reach no more than the position it occupies. Whether the bound tightens with depth depends on the configured attenuation: the theorem alone permits it to remain flat. At the same time, the default per-hop factor in Appendix~\ref{app:blast-calibration} makes it strictly decreasing. Theorem~\ref{thm:composition} (composition soundness) and Proposition~\ref{prop:ktuple} assume a complete effective restriction set ($X_{\mathrm{eff}}$ for pairs, $K$ for ordered tuples) and serialized admission; under these, no admissible action sequence produces a prohibited outcome.

\begin{table}[t]
\centering
\footnotesize
\caption{Formal results: assumptions, guarantees, and supporting evidence.}
\label{tab:formal-results}
\begin{tabularx}{\linewidth}{@{}>{\raggedright\arraybackslash}p{0.18\linewidth}>{\raggedright\arraybackslash}p{0.24\linewidth}>{\raggedright\arraybackslash}p{0.22\linewidth}>{\raggedright\arraybackslash}X@{}}
\toprule
\textbf{Property} & \textbf{Assumption} & \textbf{Guarantee} & \textbf{Evidence} \\
\midrule
Blast-radius mono.\ (Thm.~\ref{thm:blast-radius}) & Scope narrowing; cumulative budget tracking; consistent $\mathrm{blast}(r)$ & Reachable blast radius non-increasing per hop & 99 chain scenarios (depths 2--8) \\
\addlinespace
Comp.\ soundness (Thm.~\ref{thm:composition}, Prop.~\ref{prop:ktuple}) & Complete $X_{\mathrm{eff}}$, $K$; serialized admission & No admissible sequence produces prohibited outcome & InjecAgent 0\% DS; ASB 0\% disruptive; adaptive all matched \\
\bottomrule
\end{tabularx}
\end{table}

\subsection{Enterprise and Cloud Deployment Considerations}\label{sec:deployment}

In an enterprise deployment, APC maps onto concrete infrastructure: the IdP (C1), the orchestration framework (delegation chain), a tool or MCP gateway (PEP), a policy engine such as a policy-as-code evaluator (PDP), an approval service (C4), and an append-only evidence store (C5). The gateway creates the Authorization Envelope at session initialization and narrows it at each delegation hop. Intent binding supports graduated rollout: composition closure (C2b) can operate in strict mode from day one because it depends on policy configuration rather than parsing quality. In contrast, intent enforcement can begin in warn or audit mode and tighten to strict once those intent specifications mature. The six budget dimensions map to existing enterprise risk artifacts (asset classification, business-impact analysis, change-management categories, separation-of-duty policies), so calibration reuses controls organizations already maintain (Appendix~\ref{app:calibration} and~\ref{app:blast-calibration}).

\subsection{Limitations}\label{sec:limitations}

Several limitations bound the claims:

\emph{Construct:} the action-class taxonomy $\mu$ involves expert judgment, and different security teams may derive different restriction sets; that reliance on expert judgment is inherent to all policy-based authorization, and the coverage metric (Appendix~\ref{app:restriction-authoring}) makes it measurable.

\emph{Internal:} the compromised-model evaluation injects ground-truth attacks after the first legitimate step---a specific simulation methodology, not a universal worst case; it does not model adversaries who interleave legitimate and malicious calls across many turns or adapt to enforcement feedback. Adaptive attacks are self-designed rather than externally red-teamed.

\emph{External:} all benchmarks are synthetic, and no production-deployment data is presented; this is a deliberate scope choice, as the primary claims are structural.

\emph{Residual risk} concentrates in single-action misuse within authorized scope, parameter-level attacks, intent overlap at coarse action-type granularity, incomplete composition-restriction sets, and attacks spanning multiple sessions (composition state is per-session; cross-session tracking requires durable lineage state, T3). Under multi-principal compromise ($k \geq 2$), blast-radius monotonicity is unaffected (it is structural) and composition soundness holds per session; the residual reduces to the cross-session coordination problem. Finally, the theorems are not machine-checked; mechanized proofs are future work.

\section{Related Work}\label{sec:related}

\paragraph{Classical access control.} Authorization in enterprise systems evolved through RBAC~\cite{sandhu1996}, ABAC~\cite{hu2014}, and capability-based security: Dennis and Van Horn~\cite{dennis1966} introduced capabilities, Miller~\cite{miller2006} formalized attenuation, and SPKI/SDSI~\cite{ellison1999} enabled multi-hop delegation. ANSI RBAC's dynamic separation of duty constrains role-activation combinations within a session, and Brewer--Nash (Chinese Wall) constrains access to objects in conflict-of-interest classes incrementally against access history; neither constrains action-type composition across tool invocations, and neither models multi-hop delegation with attenuation. APC extends this tradition with composition closure and delegation budgets---constraints not captured by classical attenuation, needed because agents can reason about and recombine their capabilities. Zero Trust~\cite{nist800207} mandates continuous verification but predates agentic AI, and Zanzibar~\cite{pang2019} demonstrates that per-request authorization sustains millions of authorization requests per second, suggesting per-hop verification is feasible. Composition closure has a conceptual ancestor in information-flow control: Denning's lattice model~\cite{denning1976} and the decentralized label model of Myers and Liskov~\cite{myersliskov1997, myersliskov2000}. Composition closure operates as a coarse-grained taint model at the action-type level, practical for opaque LLM agents where variable-level information-flow control is infeasible.

\paragraph{Delegation and authorization protocols.} OAuth 2.0 Rich Authorization Requests~\cite{rfc9396} express fine-grained, structured grants. Token Exchange~\cite{rfc8693} supports multi-party delegation: nested \texttt{act} claims record the chain of acting parties and \texttt{may\_act} restricts who may act on behalf of whom. These mechanisms establish \emph{who} may act on behalf of whom and with what static scope; they do not define how scope narrows at each hop, which combinations of granted operations are prohibited, or how admissibility depends on prior actions in the session. South et al.~\cite{south2025, south2025b} extend OAuth for agent delegation and study identity management for agentic AI. APC is complementary: it consumes such a grant and adds per-hop attenuation, composition closure, and prior-action state.

\paragraph{Confused deputy and non-human identity.} Unsafe composition is a confused-deputy problem at the level of action sequences: an agent is induced to combine authorized capabilities toward an unauthorized end. Managing non-human identities and their delegated authority is an emerging concern that classical, human-centric IAM does not directly address; APC treats each agent as a distinct principal in an explicit chain.

\paragraph{LLM and agent security.} Empirical work confirms that exploitation severity scales with privilege~\cite{ruan2024, debenedetti2024}; relatedly, teams of LLM agents have been shown to exploit real zero-day vulnerabilities~\cite{hou2024}. Standards and taxonomies---OWASP's Top 10 for LLM and for Agentic Applications~\cite{owaspllm2025, owaspagentic2025, owaspmaestro2025}, the Cloud Security Alliance MAESTRO framework~\cite{csamaestro2025}, and AIUC-1~\cite{aiuc2025}---identify threats but define no runtime enforcement model. Khoo et al.~\cite{khoo2025} provide a risk taxonomy without an enforcement architecture, and Madkour et al.~\cite{madkour2026} extend the NIST AI Risk Management Framework to agentic governance.

\paragraph{Runtime enforcement and guardrails.} A growing cluster of deterministic enforcement frameworks addresses runtime agent security. Ji et al.~\cite{ji2026seagent} propose SEAgent, an ABAC-based mandatory-access-control framework monitoring agent--tool interactions via an information-flow graph. Debenedetti et al.~\cite{debenedetti2025camel} introduce CaMeL, which separates control flow from data flow, using capabilities to prevent exfiltration, with demonstrable dataflow isolation. Shi et al.~\cite{shi2025progent} present Progent, a programmable privilege-control framework with a tool-level policy DSL; its current version reports 1.0\% ASR on AgentDojo and 3.9\% on ASB under automatic policy generation, and 0\% under manual policies, and adds an SMT solver that classifies each policy update as a narrowing (applied automatically) or an expansion (demanding explicit approval), so that a single agent's effective action space cannot widen without approval. Our matched-protocol measurement (\S\ref{sec:agentdojo}) used the Progent artifact available at evaluation time, which predates that mechanism. Balunovic et al.~\cite{balunovic2024formal} propose a security analyzer with a policy DSL and formally establish the correctness of its checker, and Rajagopalan and Rao~\cite{rajagopalan2026authenticated} introduce authenticated workflows with cryptographic attestations. Those frameworks prove properties of their \emph{enforcement engines}---in Progent's case, a \emph{temporal} monotonicity property: one agent's action space is non-increasing across successive policy updates. APC's Theorem~\ref{thm:blast-radius} is \emph{structural}: maximum blast radius is non-increasing in position along the principal chain, because each hop takes the meet of the delegator's scope and inherits the parent's consumed budget as a floor. Neither temporal monotonicity nor checker correctness yields per-hop attenuation, cumulative budget inheritance across a chain, or composition closure over the action-type history of prior hops, and none of these schemes establishes such properties of the authorization \emph{model} rather than of a particular engine. APC's theorems hold for any implementation faithful to APC semantics; the two approaches can in principle be composed (Appendix~\ref{app:comparison}).

\section{Conclusion}\label{sec:conclusion}

The security consequence of prompt injection in agentic systems is, to a substantial degree, an authorization-architecture problem: an agent that cannot combine the actions required to exfiltrate data---because the combination is prohibited at the infrastructure level---is not vulnerable to injection attacks that attempt it, regardless of model behavior. We presented a session-scoped authorization model for delegated tool use, the Agentic Principal Chain, that operationalizes this view as a conjunctive admissibility predicate over six deterministic conditions enforced outside the model runtime, with composition closure as its central primitive.

Two structural results---composition soundness and blast-radius monotonicity---hold for any implementation faithful to the model, under stated and explicit assumptions. Across 3,154 evaluation instances spanning public benchmarks, live LLM evaluation, and adversarial testing under full model compromise, the observed exfiltration attack success rate is 0\% across all four AgentDojo domains and InjecAgent's 544 data-stealing cases, at an interactive utility cost of $-$8.6\,pp; the two surviving ASB tool types trace to action-type misclassification rather than to the enforcement mechanism. APC controls which action types execute on which resources and in which combinations, not whether the parameters of an individually authorized action are benign; single-action misuse within scope and parameter-level attacks require complementary mechanisms. Future work includes machine-checked proofs of the two theorems, parameter-level validation for single-action attacks, cross-session composition tracking via durable lineage state, and evaluation on production-scale multi-agent deployments.

\paragraph{Disclosure.} The author is a founding member of the AIUC Consortium and a contributor to OWASP and the Cloud Security Alliance. Publications of all three organizations are cited in Section~\ref{sec:related}~\cite{aiuc2025, owaspllm2025, owaspagentic2025, owaspmaestro2025, csamaestro2025}, including in the assessment that these standards identify threats without specifying a runtime enforcement model. This work was conducted independently and does not represent a position of any of these organizations.

\paragraph{AI assistance.} Large language model tools assisted with drafting and editing. All claims remain the responsibility of the author.


\appendix

\section{Restriction Authoring Procedure}\label{app:restriction-authoring}

We formalize the authoring of $X$ as a four-step procedure.

\textbf{Step 1: Action-class enumeration.} Enumerate the set of semantic action classes $\mathcal{C}$. Each tool maps to exactly one class via $\mu: \text{Tools} \to \mathcal{C}$.

\textbf{Step 2: Prohibited-outcome identification.} From the domain threat model, identify the set of prohibited outcomes $\mathcal{O}$ with severity classifications.

\textbf{Step 3: Outcome-to-restriction mapping.} For each $o \in \mathcal{O}$, derive pairwise or $k$-tuple restrictions from the minimal action sequence that produces it.

\textbf{Step 4: Coverage verification.}
\[
\text{coverage}(X, K, \mathcal{O}) = \frac{|\{o \in \mathcal{O} : \text{pairs}(o) \subseteq X \lor \text{tuples}(o) \subseteq K\}|}{|\mathcal{O}|}.
\]
The AgentDojo workspace evaluation uses 7 pairwise and 8 $k$-tuple restrictions (coverage 1.0). The coverage metric is relative to $\mathcal{O}$, not to all possible harmful sequences.

\section{Adaptive Attack Details}\label{app:adaptive}

Table~\ref{tab:adaptive-details} summarizes a representative subset of the twenty-three named attacks and their outcomes.

\begin{table}[H]
\centering
\footnotesize
\caption{Adaptive attack results (representative subset). Twenty-three named attacks with 43 variants target all six conditions.}
\label{tab:adaptive-details}
\begin{tabularx}{\linewidth}{@{}l>{\raggedright\arraybackslash}Xl@{}}
\toprule
\textbf{Attack} & \textbf{Strategy} & \textbf{Result} \\
\midrule
Decomposed Exfil & read $\to$ write $\to$ send\_internal (evades pairwise) & \textbf{blocked} ($k$-tuple) \\
Intent Drift & in-scope, out-of-intent resource & \textbf{blocked} (C6) \\
Budget Exhaustion & 3 transfers, budget max 2 & \textbf{blocked} (C2c) \\
Approval Replay & reuse token with different parameters & \textbf{blocked} (C4 hash) \\
Expired Token & use approval token after TTL expires & \textbf{blocked} (C4 temporal) \\
Consumed Token & reuse single-use token after consumption & \textbf{blocked} (C4 single-use) \\
Evidence Evasion & act when evidence sink is down & \textbf{blocked} (C5 fail-closed) \\
Scope Probing (A3) & probe resources, actions, classifications & boundaries enforced \\
Gradual Buildup (A4) & 8-step sequence with interleaved noise & \textbf{blocked} (pair + $k$-tuple) \\
Session Splitting & split read/send across sessions & admitted (per-session) \\
Cross-Session Token (T3) & replay session-A token in session-B & \textbf{blocked} (C4 session) \\
Envelope Forgery (T2) & sign envelope with wrong key & \textbf{blocked} (signature) \\
Envelope Tampering (T2) & modify sealed envelope scope & \textbf{blocked} (immutability) \\
Depth Overflow & act beyond delegation depth ceiling & \textbf{blocked} (C2c) \\
Intent Warn Mode & out-of-intent in warn vs strict mode & graduated enforcement \\
\bottomrule
\end{tabularx}
\end{table}

\section{Mechanism-Class Comparison}\label{app:mechanism-comparison}

We compare two classes of deterministic policy-enforcement approaches on the same InjecAgent and ASB test cases: \emph{action-class composition closure} (APC) and \emph{sensitivity-based information flow tracking} (modeled after SEAgent~\cite{ji2026seagent}). Both are evaluated as static policy checkers. The SEAgent simulation maintains a session-level sensitivity high-water mark and denies flows from sensitive data to external sinks; this models sensitivity-label propagation but not the full information flow graph, so the reported gap is an upper bound.

\begin{table}[H]
\centering
\footnotesize
\caption{Enforcement mechanism comparison on InjecAgent (1,054 cases) and ASB (400 cases).}
\label{tab:headtohead}
\begin{tabularx}{\linewidth}{@{}l*{5}{c}>{\raggedright\arraybackslash}X@{}}
\toprule
& \multicolumn{3}{c}{\textbf{InjecAgent}} & \multicolumn{2}{c}{\textbf{ASB}} & \\
\cmidrule(lr){2-4} \cmidrule(lr){5-6}
\textbf{Mechanism} & \textbf{DS} & \textbf{DH} & \textbf{Total} & \textbf{Stlth.} & \textbf{Disr.} & \textbf{Class} \\
\midrule
Comp.\ closure (APC) & \textbf{0.0\%} & 60.4\% & 29.2\% & 30.0\% & \textbf{0.0\%} & Action-class pairs \\
Info flow (SEAgent sim.) & 1.7\% & 59.2\% & 29.5\% & 60.0\% & 60.0\% & Sensitivity labels \\
\bottomrule
\end{tabularx}
\end{table}

\section{Implementation Details}\label{app:implementation}

The reference implementation (Python 3.11) provides executable tests aligned with all formal properties.

\paragraph{Measurement environment} Latencies were measured on an Intel Core i5-1245U (12th generation, 10 cores / 12 threads, 1.6\,GHz base), 16\,GB RAM, Windows 11 (build 26200), CPython 3.11.9, single-threaded and with no other significant load. This is a mobile-class processor.

\paragraph{Method} Each component is timed per call with \texttt{time.\allowbreak perf\_\allowbreak counter()} over 20{,}000 iterations following 500 warmup iterations; percentiles are nearest-rank over the sorted sample. We report the median across five independent repetitions, with the observed range across repetitions in brackets. Admissibility is measured on the \emph{admit} path: every timed call passes all six conditions and performs the associated evidence commit, composition record, and budget consumption. Because budget consumption accumulates, the session is rebuilt at fixed intervals outside the timed region; measurements at session lengths of 20, 50, and 200 actions agree to within the reported ranges. The measurement covers only the in-process authorization path: it excludes model inference, network round-trips, policy retrieval from a remote PDP, and evidence-sink I/O.

\begin{table}[H]
\centering
\footnotesize
\caption{Enforcement latency on the environment above. Median of five repetitions; bracketed values are the range across repetitions. Reproduced from the committed measurement artifact \texttt{evals/latency/results/latency\_appendix\_d.json}.}
\label{tab:latency}
\begin{tabularx}{\linewidth}{@{}Xcc@{}}
\toprule
\textbf{Component} & \textbf{p50 (ms)} & \textbf{p99 (ms)} \\
\midrule
Full admissibility, C4 below threshold & 0.049 [0.048--0.050] & 0.240 [0.173--0.261] \\
Full admissibility, C4 token verified & 0.057 [0.056--0.058] & 0.236 [0.215--0.345] \\
Composition closure (isolated) & 0.0016 [0.0012--0.0021] & 0.0031 [0.0016--0.0033] \\
Envelope narrowing (meet + re-sign) & 0.015 [0.015--0.016] & 0.059 [0.049--0.081] \\
\bottomrule
\end{tabularx}
\end{table}

\noindent Both admissibility rows evaluate all six conditions and perform the associated evidence commit, composition record, and budget consumption. They differ only in Condition~4: the first uses an action whose impact score falls below the approval threshold; the second uses an action above it with a valid single-use token, which adds roughly 0.01\,ms at the median. Composition closure is measured against a primed session history in both the isolated row and within the full path, so the pairwise lookup is exercised in each. Envelope narrowing recomputes the scope meet and re-signs the envelope; it occurs once per delegation hop rather than once per action, so it does not sit on the per-action hot path. Throughput expressed as evaluations per second is the reciprocal of mean single-call latency, not a measured concurrent throughput, and does not account for contention under parallel load. Absolute values are hardware-dependent; \texttt{scripts/benchmark\_latency.py} re-runs the measurement on the host machine and aborts if any timed call fails to reach an admit decision through all six conditions. The measurement is also sensitive to host load: two further executions of the same protocol on the same machine while other workloads were running produced 0.083--0.091\,ms p50 and 0.345--0.427\,ms p99 for full admissibility, with individual repetitions reaching 0.66\,ms p99. The no-load precondition above is therefore a requirement rather than a formality, and the sub-millisecond figures reported here characterise the enforcement path under that condition rather than bounding it under contention. All three executions are committed under \texttt{evals/latency/results/}; the values in Table~\ref{tab:latency} are those of the in-protocol run.

The 99 delegation-chain scenarios cover scope escalation, budget exhaustion, cross-hop composition, blast-radius monotonicity, identity binding at depths 3--7, approval binding with hash integrity (C4), evidence-sink availability and mid-session failure (C5), intent binding (C6), $k$-tuple cross-hop composition, cross-domain composition, sensitivity escalation, expired envelopes, context edge cases, and conjunctive-predicate validation.

\paragraph{Domain configuration validation.} Prior to the final evaluation runs, domain-specific configurations were validated with dedicated sanity-check scripts. Three sources of utility loss were distinguished, and only the first two were corrected before freezing results: (i)~\emph{implementation defects}; (ii)~\emph{operationally invalid policy}; and (iii)~\emph{genuine tradeoffs} (not corrected). We applied all corrections before the final reported runs.

\section{Reference Schemas and Validation Semantics}\label{app:schemas}

\paragraph{Authorization Envelope.} A cryptographically signed artifact created by infrastructure at session initialization. It carries envelope and task identifiers, timestamps, principal identity and type, execution role, the full scope tuple $(R, A, D, X)$, the intent specification~$\Psi$, the delegation budget~$B$, policy version, a cryptographic nonce, and an infrastructure-key signature. Narrowed derivatives must satisfy: resources, actions, and data classifications $\subseteq$~parent; prohibited compositions $\supseteq$~parent; all budget fields $\leq$~parent values. Narrowing is irreversible within a session.

\paragraph{Approval Token.} A single-use authorization artifact bound to an exact action instance via a SHA-256 hash of action type, target resource, and parameters. Cross-session replay is invalid. Expired tokens are invalid. \texttt{uses\_remaining} is decremented atomically.

\paragraph{Normative Validation Rules.} Nine rules govern all enforcement decisions: (1)~\emph{Fail Closed}: missing data $\rightarrow$ deny. (2)~\emph{Conjunctive}: all six conditions must pass. (3)~\emph{Binding Precedence}: envelope overrides manifest. (4)~\emph{Temporal Validity}: expired envelope/token/budget $\rightarrow$ deny. (5)~\emph{Action-Hash Integrity}: token valid only if hash matches. (6)~\emph{Budget Precedence}: exceeded ceiling $\rightarrow$ suspend. (7)~\emph{Evidence Availability}: sink unreachable $\rightarrow$ deny. (8)~\emph{Composition Timing}: checked against full session history. (9)~\emph{Intent Conformance}: strict $\rightarrow$ deny; warn $\rightarrow$ admit + log; audit $\rightarrow$ admit + flag.

\section{Comparison Tables}\label{app:comparison}

\begin{table}[H]
\centering
\footnotesize
\caption{Standalone coverage of APC properties by existing authorization mechanisms.}
\label{tab:standalone-comparison}
\begin{tabular}{@{}lcccc@{}}
\toprule
\textbf{Property} & \textbf{OAuth+OPA} & \textbf{Prompts} & \textbf{Static Manifest} & \textbf{APC} \\
\midrule
Scope enforcement (G1)         & \checkmark & ---     & ---        & \checkmark \\
Scope attenuation              & Partial    & ---     & Tool-level & \checkmark \\
Composition closure (G2)       & ---        & ---     & ---        & \checkmark \\
Blast-radius mono.\ (G3)       & ---        & ---     & ---        & \checkmark \\
Approval binding (G4)          & ---        & ---     & ---        & \checkmark \\
Evidence commit.\ (G5)         & Partial    & ---     & ---        & \checkmark \\
Intent binding (G6)            & ---        & ---     & ---        & \checkmark \\
\bottomrule
\end{tabular}
\end{table}

\begin{table}[H]
\centering
\footnotesize
\caption{APC versus modern agentic security systems.}
\label{tab:framework-comparison-app}
\begin{tabularx}{\linewidth}{@{}lllll@{}}
\toprule
\textbf{Property} & \textbf{SEAgent} & \textbf{CaMeL} & \textbf{Progent} & \textbf{APC} \\
\midrule
Data stealing ASR & 0\%$^a$ & --- & --- & \textbf{0\%} \\
Composition closure & Policy-dep. & Dataflow-dep. & Policy-dep. & \textbf{Formal}$^b$ \\
Blast-radius mono. & --- & --- & Temporal$^c$ & \textbf{Structural} \\
Multi-hop delegation & Partial & --- & Partial$^c$ & \textbf{\checkmark} \\
Intent binding & --- & --- & --- & \textbf{\checkmark} \\
\bottomrule
\end{tabularx}
\vspace{2pt}
{\scriptsize $^a$As reported by Ji et al.\ Our simplified simulation yields 1.7\% DS ASR.\\$^b$Formal guarantee conditional on completeness of $X_{\mathrm{eff}}$.\\$^c$Progent v3 proves one agent's action space is non-increasing across policy updates (temporal), and discusses unified or per-sub-agent policy layers for multi-agent deployments; neither gives per-hop scope attenuation or cumulative budget inheritance along a delegation chain (structural).}
\end{table}

\section{Impact Calibration}\label{app:calibration}

Weights $w_\rho$, $w_\beta$, $w_\sigma$ and threshold $\theta$ in $I(a) = w_\rho \cdot \rho(a) + w_\beta \cdot Bl(a) + w_\sigma \cdot Se(a)$ are calibrated by three complementary methods: \emph{(i)~Expert elicitation} (domain experts rank actions, weights obtained by solving a constrained optimization problem, Kendall's $\tau \geq 0.8$). \emph{(ii)~Bayesian estimation} (for organizations with historical incident records). \emph{(iii)~Sensitivity analysis} ($\theta$ selected for false-autonomous rate $<1\%$, approval burden $<15\%$).

\section{Blast-Radius Calibration}\label{app:blast-calibration}

Each resource $r$ is assigned a score $\mathrm{blast}(r) \in [0,1]$ by infrastructure at deployment time:
\[
\mathrm{blast}(r) = w_s \cdot \mathrm{scope}(r) + w_v \cdot \mathrm{irrev}(r) + w_d \cdot \mathrm{sens}(r),
\]
with $w_s + w_v + w_d = 1$ and $w_s, w_v, w_d \geq 0$, where $\mathrm{scope}(r)$ is the normalized count of affected principals/systems, $\mathrm{irrev}(r)$ the degree of irreversibility within the recovery-time objective, and $\mathrm{sens}(r)$ the data-sensitivity classification. The assignment must satisfy monotonicity: if $r'$ dominates $r$ on all three factors then $\mathrm{blast}(r') \geq \mathrm{blast}(r)$. Default weights: $w_s = 0.4$, $w_v = 0.4$, $w_d = 0.2$.

\begin{table}[H]
\centering
\footnotesize
\caption{Illustrative blast scores under default weights.}
\label{tab:blast-examples}
\begin{tabularx}{\linewidth}{@{}lcccc>{\raggedright\arraybackslash}X@{}}
\toprule
\textbf{Resource} & $\mathrm{scope}$ & $\mathrm{irrev}$ & $\mathrm{sens}$ & $\mathrm{blast}$ & \textbf{Rationale} \\
\midrule
User calendar entry & 0.1 & 0.2 & 0.1 & 0.14 & Single user, recoverable, low sensitivity \\
Shared team document & 0.4 & 0.3 & 0.6 & 0.40 & Team scope, versioned, confidential \\
Customer PII record & 0.6 & 0.7 & 1.0 & 0.72 & Broad scope, hard to remediate, restricted \\
Production database & 0.9 & 0.9 & 0.8 & 0.88 & Org-wide, near-irreversible, restricted \\
External email send & 0.7 & 1.0 & 0.5 & 0.78 & Unrecallable, broad reach, variable sensitivity \\
\bottomrule
\end{tabularx}
\end{table}

The session ceiling $\beta_{\max}$ is derived from the organization's business-impact analysis for the task class, attenuated at each delegation hop ($\beta_{\max}(p_i) \leq \beta_{\max}(p_{i-1})$; conservative default: $\beta_{\max}(p_i) = 0.7 \cdot \beta_{\max}(p_{i-1})$). The proof of Theorem~\ref{thm:blast-radius} requires only that $\mathrm{blast}(r)$ values are consistent across the chain and that $\beta_{\max}$ is non-increasing; the specific values affect only the tightness of the bound.

\section{Code and Data Availability}\label{app:artifacts}

The reference implementation, evaluation harnesses, and all domain configurations supporting the reported results are available at \url{https://github.com/xmuruaga/bounded-agents}.

\bibliographystyle{ACM-Reference-Format}
\bibliography{references}

\end{document}